\documentclass[lettersize,journal]{IEEEtran}

\usepackage{amsmath,amsfonts,amssymb}  % 数学公式与符号
\usepackage{amsthm}                     % 定理环境支持

\usepackage{algorithmic}
\usepackage{algorithm}

\usepackage{array}
\usepackage{multirow}
\usepackage{tabularx} % 导言区添加
\usepackage{booktabs}

\usepackage{graphicx}
\usepackage[caption=false,font=normalsize,labelfont=sf,textfont=sf]{subfig} % 与IEEE官方模板一致
\usepackage{pgfplots}
\pgfplotsset{compat=1.17}
\usepackage{tikz}

\usepackage{textcomp}
\usepackage{stfloats}
\usepackage{url}
\usepackage{verbatim}
\usepackage[colorlinks=true,linkcolor=blue,citecolor=blue,urlcolor=blue]{hyperref}

\usepackage{cite}

\newtheorem{theorem}{Theorem}[section]
\newtheorem{definition}{Definition}[section]
\newtheorem{example}{Example}[section]

\begin{document}
\begin{sloppypar}

\title{GCAO: Group-driven Clustering via Gravitational Attraction and Optimization}

\author{Qi~Li\thanks{Qi Li is with the School of Information Science and Technology, 
Beijing Forestry University, Beijing 100083, China 
(e-mail: liqi2024@bjfu.edu.cn). 
This work was supported by the Fundamental Research Funds for the Central Universities 
(No. XJJSKYQD202536).}
and
Jun~Wang\textsuperscript{*}\thanks{Jun Wang is with the College of Forestry, 
Beijing Forestry University, Beijing 100083, China 
(e-mail: wangjun1704@bjfu.edu.cn). 
(*Corresponding author.)}}

% The paper headers
\markboth{}%
{Li \MakeLowercase{\textit{et al.}}: GCAO: Group-driven Clustering via Gravitational Attraction and Optimization}

\IEEEpubid{}

\maketitle

\begin{abstract}
Traditional clustering algorithms often struggle with high-dimensional and non-uniformly distributed data, where low-density boundary samples are easily disturbed by neighboring clusters, leading to unstable and distorted clustering results. To address this issue, we propose a \textbf{Group-driven Clustering via Gravitational Attraction and Optimization (GCAO)} algorithm. GCAO introduces a group-level optimization mechanism that aggregates low-density boundary points into collaboratively moving groups, replacing the traditional point-based contraction process. By combining local density estimation with neighborhood topology, GCAO constructs effective gravitational interactions between groups and their surroundings, enhancing boundary clarity and structural consistency. Using groups as basic motion units, a gravitational contraction strategy ensures globally stable and directionally consistent convergence. Experiments on multiple high-dimensional datasets demonstrate that GCAO outperforms 11 representative clustering methods, achieving average improvements of \textbf{37.13\%}, \textbf{52.08\%}, \textbf{44.98\%}, and \textbf{38.81\%} in NMI, ARI, Homogeneity, and ACC, respectively, while maintaining competitive efficiency and scalability. These results highlight GCAO's superiority in preserving cluster integrity, enhancing boundary separability, and ensuring robust performance on complex data distributions.
\end{abstract}

\begin{IEEEkeywords}
Clustering algorithms, Gravitational contraction, Group movement, High-dimensional data
\end{IEEEkeywords}

\section{INTRODUCTION}

\IEEEPARstart{C}{lustering} is a core unsupervised learning technique, aiming to partition a dataset into several meaningful clusters based solely on the intrinsic similarity among samples, such that samples within a cluster are highly similar while samples across clusters differ significantly \cite{jain2010data, xu2015comprehensive}. As a cornerstone in fields such as data mining, pattern recognition, image processing, and bioinformatics, clustering analysis plays a crucial role in revealing the latent structures and high-dimensional distribution patterns of data.

\textbf{Research Background and Motivation.}
In recent years, with the continuous enhancement of data acquisition capabilities, modern datasets exhibit characteristics such as large scale, high dimensionality, complex structures, and mixed densities. Traditional clustering algorithms, such as K-Means \cite{KMeans2020} and DBSCAN \cite{DBSCAN2017}, although performing well in specific tasks, often rely on strict parameter assumptions (e.g., $K$ value or $\epsilon$ radius), making it difficult to maintain stable clustering performance in complex distributions.

To address these limitations, researchers have proposed clustering methods based on the concepts of \textbf{Density Contraction} and \textbf{Potential Minimization}, such as Mean-Shift \cite{MeanShift2002} and Density Peaks Clustering (DPC) \cite{DPC2014}. The former iteratively updates the positions of samples within a density gradient field, gradually converging to mode points in high-density regions; the latter identifies density peaks as cluster centers in a static manner by computing the local density of each sample and its minimum distance to higher-density points. Although these methods perform well in capturing non-convex structures and automatically determining cluster centers, they primarily focus on individual-level density distributions and distance relationships, lacking modeling of neighborhood interactions and overall potential evolution. Therefore, they still exhibit limitations when handling non-uniform densities or complex geometric structures.

\begin{figure}[htbp]
    \centering    \includegraphics[width=0.15\textwidth]{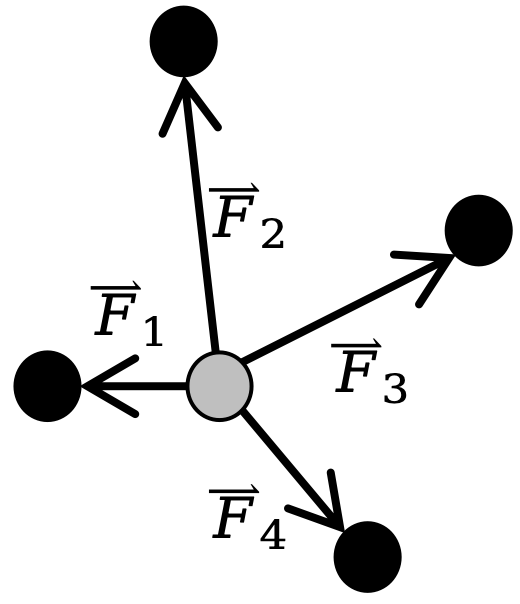}
    \caption{Visualization of Gravitational Mechanism}
    \label{fig:ils_gt}
\end{figure}

\textbf{Introduction and Development of Gravitational Mechanism.}
In recent years, the concept of \textbf{Gravitational Mechanism}-based clustering has gradually emerged, providing a more physically interpretable framework for contraction-based algorithms. As shown in Fig.~\ref{fig:ils_gt}, the basic idea originates from Newtonian Gravitation, where the "attractive force" between data points can drive samples to aggregate along the potential direction, forming stable cluster structures at gravitational equilibrium points. Representative studies include the \textbf{Newtonian Clustering} \cite{Newtonian2007} model, which defines a gravitational function based on data point mass and distance to achieve continuous potential descent and convergence; the \textbf{Herd Clustering} \cite{Herd2014} model, which introduces a group behavior mechanism, allowing samples to self-organize into clusters via "group movement"; and the \textbf{SBCA (Swarm-based Clustering Algorithm)} \cite{SBCA2005}, which further combines swarm intelligence with gravitational attraction, enhancing convergence speed and noise robustness.

Building on this foundation, a series of improved algorithms based on \textbf{Gravitational Contraction} have been proposed. For example, \textbf{HIBOG} \cite{HIBOG2021} introduces a quasi-gravitational mechanism between each object and its $k$ nearest neighbors, attracting objects toward neighbors to achieve gradual cluster aggregation; \textbf{HIAC}\cite{HIAC2023} further refines the adjacency strategy of HIBOG by distinguishing effective neighbors from ineffective ones, retaining gravitational effects only among effective neighbors, thereby enhancing boundary recognition and overall clustering accuracy; \textbf{HIACSP}\cite{HIACSP2025} introduces a shortest-path-based distance metric, encouraging the $k$ nearest neighbors to concentrate toward the same cluster core, retaining attraction only among intra-cluster objects without requiring an attraction threshold, effectively avoiding micro-clusters and boundary point shifts. The common feature of these algorithms is the transformation of "individual movement" into "neighborhood collaborative contraction" through the gravitational mechanism, reflecting physical consistency and global adaptiveness in the clustering process.

\textbf{Current Research Status and Challenges.}
Although gravitational mechanism-based clustering methods have achieved significant progress in maintaining structural continuity and noise robustness, existing algorithms are mostly limited to modeling interactions at the "point-to-point" or "local neighborhood" level, lacking a systematic characterization of \textbf{Group-level Cooperation}. The main issues include: (1) \textbf{Structural Disruptiveness}: isolated points in boundary or sparse regions are easily misclustered due to local gravitational interference, leading to cluster splitting and blurred boundaries; (2) \textbf{Parameter Sensitivity}: parameters such as attraction strength or neighborhood radius significantly affect convergence results, with insufficient adaptive mechanisms for different datasets; (3) \textbf{Convergence Instability}: single-point interactions driven by gravity may fall into local equilibrium, causing cluster center drift or oscillation.

\textbf{Proposed Algorithm.}
To overcome the above limitations, this paper proposes a \textbf{Group-driven Clustering via Gravitational Attraction and Optimization (GCAO)} algorithm, achieving a transition from "individual contraction" to "group collaborative contraction." The core ideas include: (1) constructing local collaboratively moving groups based on \textbf{truncated radius} and \textbf{local density}, and computing the inter-group interaction force using a step coefficient $\lambda$ to iteratively update groups, enabling low-density points in boundary regions to move collaboratively under gravitational influence, thereby preserving the integrity and continuity of cluster structures; (2) designing a \textbf{group gravitational driving mechanism} to dynamically optimize the contraction process during gravitational interactions, ensuring the stability and consistency of clustering results.

This "group-level gravitational contraction" mechanism effectively combines the advantages of density guidance and physical attraction, mathematically enhancing the convergence stability of the algorithm and structurally strengthening the coherence of cluster boundaries, thereby demonstrating superior clustering performance on complex distribution data.

\textbf{Objectives and Challenges.}
This work aims to design a highly robust and interpretable gravitational contraction optimization model that can achieve stable and accurate clustering in multi-density, nonlinearly distributed data environments. Achieving this goal faces two key challenges:
\begin{itemize}
    \item \textbf{Challenge 1: Adaptive neighborhood range and rational group construction.} Dynamically defining the gravitational neighborhood based on local density variations and identifying effective collaboratively moving groups are prerequisites for maintaining structural integrity;
    \item \textbf{Challenge 2: Stability and global convergence of group gravitational contraction.} Balancing the coupling relationship among attraction strength, step coefficient $\lambda$, and the number of iterations to achieve a dynamic optimum between global convergence and boundary equilibrium is the core difficulty in designing the gravitational mechanism.
\end{itemize}

\textbf{Contributions.}  
The main contributions of this paper can be summarized as follows:

\begin{enumerate}
    \item We propose a GCAO algorithm, which aggregates low-density points at cluster boundaries into collaboratively moving groups, replacing the traditional point-level contraction mechanism. This "group-level optimization" significantly enhances the continuity and structural preservation of cluster boundaries, effectively suppressing random deviations in low-density regions and inter-cluster interference. Compared with traditional single-point gravitational models, this mechanism demonstrates higher stability and robustness under complex distributions and high-noise scenarios.

    \item We construct \textbf{collaboratively moving groups}, dynamically adjusting the interaction range based on local density estimation and neighborhood topology information to achieve adaptive group formation. This strategy retains only the effective gravitational interactions between the group and its exterior, avoiding cluster shifts and boundary misclassification, significantly reducing dependence on fixed parameters and improving the preservation of cluster structure information.

    \item We design a \textbf{group gravitational optimization mechanism}, performing collaborative contraction with groups as the basic motion units, and theoretically proving the consistency of contraction directions and global convergence. Extensive experimental results show that this mechanism significantly outperforms existing methods on multiple high-dimensional real-world datasets: compared with 11 representative algorithms, GCAO achieves average improvements of approximately \textbf{37.13\%}, \textbf{52.08\%}, \textbf{44.98\%}, and \textbf{38.81\%} in NMI, ARI, Homogeneity, and ACC, respectively, while maintaining high computational efficiency and good scalability.
\end{enumerate}

\textbf{Organization of the Paper.}
Section II presents related work; Section III introduces relevant definitions and research background; Section IV details the GCAO algorithm, including density estimation, group collaboration, and gravitational response mechanisms; Section V covers experimental design, comparative experiments, ablation studies, and hyperparameter discussions; Section VI concludes the paper and outlines future research directions.

\section{Related Work}
Existing clustering algorithms are typically designed based on distance, density, or graph-theoretic features. Representative methods include partition-based clustering (e.g., K-Means), density-based clustering, hierarchical clustering, and graph-based methods, each suitable for different data distributions and task scenarios.

DPC is one of the widely applied density-based clustering methods in recent years. Its core idea assumes that cluster centers generally have high local density and are far from other high-density points. DPC computes the local density of each sample and its minimum distance to higher-density points to automatically select cluster centers and partition samples. This method can identify clusters of arbitrary shapes without predefining the number of clusters and performs well on non-convex structured data. However, DPC is sensitive to parameters such as truncation distance and is easily affected by outliers in high-dimensional or noisy environments, leading to cluster center drift and decreased clustering consistency.

Unlike density peak methods, gravity-based clustering draws inspiration from Newtonian gravitation in physics, simulating attractive and repulsive forces among data points to move samples along gravitational potential gradients and eventually form clusters at force equilibrium. Classical methods such as \textbf{Newtonian Clustering} \cite{Newtonian2007} define a gravitational function based on data point mass and distance, achieving continuous potential descent and convergence; \textbf{Herd Clustering} \cite{Herd2014} introduces a group collaboration mechanism, allowing samples to self-organize into clusters through collective movement; \textbf{SBCA} \cite{SBCA2005} further combines swarm intelligence with gravitational attraction to enhance convergence speed and noise robustness. Although these methods exhibit good physical interpretability and adaptiveness, they still face challenges in high-dimensional spaces, including the curse of dimensionality, high computational complexity, and noise sensitivity.

To overcome the limitations of traditional gravitational clustering, recent studies have attempted to combine swarm intelligence optimization with gravitational mechanisms, introducing global search and local adaptive strategies to enhance clustering stability and accuracy. Among these, improved algorithms based on gravitational contraction, such as \textbf{HIBOG} \cite{HIBOG2021}, \textbf{HIAC} \cite{HIAC2023}, and \textbf{HIACSP} \cite{HIACSP2025}, incorporate neighborhood constraints, shortest-path distances, and group collaboration mechanisms to enhance boundary recognition and mitigate micro-cluster shift issues, significantly improving clustering accuracy and convergence efficiency. However, these methods generally suffer from high computational cost, sensitivity to hyperparameters, and limited model interpretability, restricting their applicability to high-dimensional complex datasets.

In summary, existing methods face three core challenges: first, single-point gravitational interactions in boundary regions are easily affected by noise, causing excessive inter-cluster forces that disrupt cluster structures; second, most methods use single points as the basic unit, ignoring collaborative movement among neighbors, resulting in unstable migration directions for boundary points; third, in high-dimensional complex data, computational demands and hyperparameter sensitivity limit practical applicability. To address these issues, this paper proposes \textbf{GCAO}, which introduces a group collaborative response mechanism within the traditional gravitational clustering framework. Low-density boundary points and their neighboring points are formed into collaborative groups for unified gravitational contraction. This strategy not only preserves local structural consistency and effectively suppresses excessive forces on inter-cluster boundary points, making cluster boundaries clearer, but also enhances global clustering stability, robustness, and adaptability to high-dimensional data, achieving superior clustering performance and model interpretability in complex data distributions.

\section{Preliminary}

\subsection{Problem Definition}
Given a clustering dataset $\mathcal{X} = \{\mathbf{x}_1, \mathbf{x}_2, \dots, \mathbf{x}_n\}$, where each data point $\mathbf{x}_i \in \mathbb{R}^d$ represents the $i$-th sample in a $d$-dimensional feature space, this work aims to reshape the data distribution structure by introducing a group-level gravitational contraction mechanism, using groups as the basic units. Samples within the same cluster are pulled toward the cluster center in feature space, thereby reducing the interference of noise and outliers, enhancing intra-cluster compactness, and expanding inter-cluster separation. The ultimate goal is to improve the separability and robustness for subsequent clustering tasks.

Specifically, the research problem can be defined as:
\begin{itemize}
    \item \textbf{Input:} The original dataset $\mathcal{X}$ and its cluster partition $C = \{\mathbf{C}_1, \mathbf{C}_2, \dots, \mathbf{C}_k\}$, where $\mathbf{C}_p$ denotes the data subset of the $p$-th cluster.
    \item \textbf{Output:} The dataset after gravitational contraction optimization $C' = \{\mathbf{C}'_1, \mathbf{C}'_2, \dots, \mathbf{C}'_k\}$, where each cluster exhibits enhanced internal compactness and significantly increased inter-cluster separation, providing improved input for subsequent clustering algorithms.
\end{itemize}

\subsection{Research Background}

In physics, gravity is the fundamental force describing the mutual attraction between two objects with mass, classically expressed by Newton's law of universal gravitation:
\begin{equation}
F = G \cdot \frac{m_1 m_2}{r^2},
\end{equation}
where $m_1$ and $m_2$ denote the masses of the two objects, $r$ is the distance between them, and $G$ is the gravitational constant. This formula indicates that the gravitational force is proportional to mass and inversely proportional to the square of the distance.

Inspired by this principle, researchers analogize data points as "massive" particles, using gravitational models to simulate interactions among data, achieving self-organization and aggregation during clustering. For example, algorithms such as DPCG~\cite{HIACSP2025}, HIBOG~\cite{HIBOG2021}, and HIAC~\cite{HIAC2023} achieve approximately \textbf{10\%–30\%} improvement in clustering accuracy compared to traditional DPC across multiple datasets, demonstrating the effectiveness of gravity-based clustering in enhancing intra-cluster compactness and inter-cluster separability.

However, traditional single-point gravitational mechanisms still face limitations when handling non-uniformly distributed or complex-shaped data: samples in low-density regions are easily over-attracted by points from other clusters, resulting in blurred cluster boundaries and disrupted local structures, thereby affecting overall clustering performance. To address this, this paper proposes a \textbf{Group-wise Gravitational Contraction (GCAO)} mechanism, introducing group-level collaborative gravitational effects. While preserving low-density structural information within original clusters, this approach effectively suppresses abnormal attraction effects, improving boundary clarity and structural stability in clustering. Extensive experimental results indicate that GCAO significantly outperforms existing gravitational clustering algorithms in both accuracy and robustness, validating the effectiveness and advancement of the group-wise gravitational contraction concept.

\section{PROPOSED METHOD: GCAO}

\begin{figure*}[htbp]
    \centering
    \includegraphics[width=0.7\textwidth]{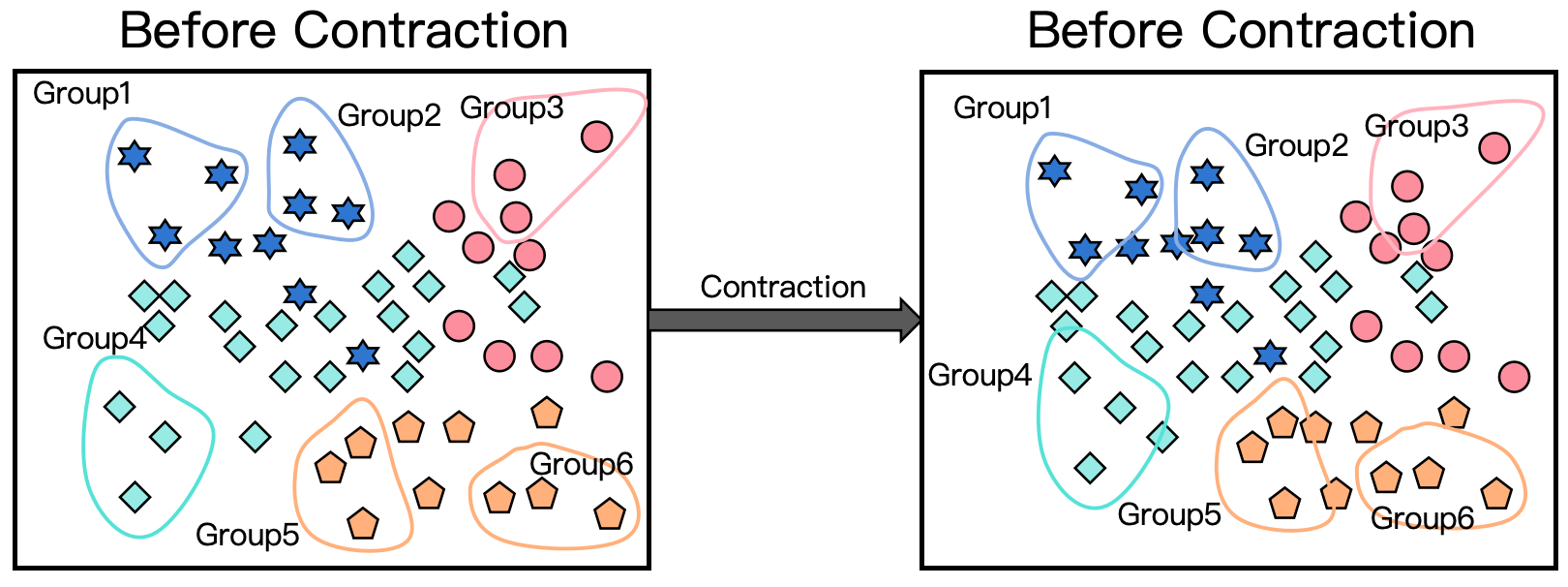}
    \caption{Overview of the proposed GCAO algorithm. }
    \label{fig:contraction}
\end{figure*}

\subsection{Overview}
This paper proposes a collaborative clustering data optimization algorithm (GCAO) that integrates local density estimation with gravitational mechanisms. Fig. \ref{fig:contraction} illustrates the overall workflow and effect of the algorithm:  
In the initial stage, the algorithm forms several collaboratively moving groups using local density estimation and neighborhood topology relationships (GFP—Group Formation Process), shown as Groups 1–6 on the left side of the figure. Subsequently, based on the Gravitational Optimization Process (GOP), the member points of each group move collaboratively toward the cluster center, contracting sample distances and optimizing cluster structures, as shown on the right side of the figure. Compared with traditional single-point gravitational clustering methods, this algorithm effectively reduces the risk of intra-cluster structure splitting while enhancing the clarity of inter-cluster boundaries, making it particularly suitable for high-dimensional, non-uniformly distributed, or complex-structured clustering data.

\subsection{GFP}

\textbf{Motivation.}  
Traditional gravitational mechanism-based clustering optimization methods typically rely on a \emph{single-point gravitational model} to update data point positions, aiming to enhance intra-cluster compactness and inter-cluster separation. However, when facing \emph{non-uniformly distributed} or complex-boundary data, this mechanism exhibits significant limitations. The key issue lies in the boundary regions between clusters: in these areas, data points are often sparsely distributed, and the attraction from intra-cluster points to boundary points is insufficient to form stable constraints. As a result, when there are nearby points from other clusters within the boundary point's neighborhood, the resultant force direction can be dominated by these inter-cluster points, as illustrated in Fig.~\ref{fig:sg_f}, where $\vec{F_1}$, $\vec{F_3}$, and $\vec{F_4}$ originate from intra-cluster points, while $\vec{F_2}$ comes from an inter-cluster point. In other words, without sufficient intra-cluster constraints, boundary points are easily pulled away by inter-cluster forces, disrupting local structures and gradually blurring inter-cluster boundaries.

\begin{figure}[htbp]
    \centering    \includegraphics[width=0.25\textwidth]{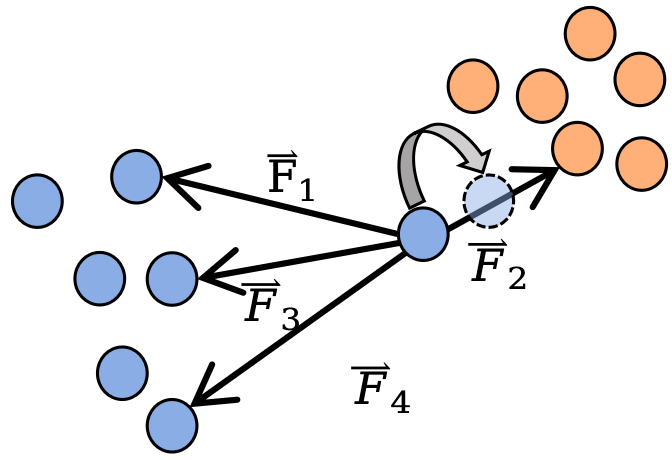}
    \caption{An example of inter-cluster forces affecting a low-density boundary point}
    \label{fig:sg_f}
\end{figure}

To address the above issues, this paper proposes a \textbf{Collaboratively Moving Group Mechanism} based on \emph{local density estimation}. In data distributions, high-density points are typically located in the cluster core regions, where they are less affected by inter-cluster attraction and thus exhibit more stable structural characteristics. In contrast, low-density points are often distributed at cluster boundaries or sparse regions, where they are easily attracted by other clusters, leading to blurred boundaries and structural distortion. To tackle this phenomenon, low-density points are adaptively identified via \textbf{local density estimation}, and several \textbf{collaboratively moving groups} are constructed around them. Each group consists of low-density points and several of their neighboring data points. Through collaborative and antagonistic interactions among groups, boundary points are guided to maintain dynamic equilibrium within the overall gravitational field. This approach effectively strengthens clustering stability and boundary separability in low-density regions, prevents boundary points from being excessively attracted by inter-cluster points, and simultaneously reduces computational complexity, enhancing both algorithm stability and global robustness.

Data points belonging to the same group are referred to as member points. Unlike single-point force updates, the movement of a group is determined collectively by the resultant force of all its member points. This design has two advantages: first, the interference from a single inter-cluster point is significantly weakened in the group resultant force, preventing local inter-cluster attraction from dominating low-density points; second, the resultant force direction more accurately reflects the overall structural trend within the cluster, enabling the group to maintain structural consistency during movement. Through this group-based collective motion, the GOP mechanism effectively mitigates the risk of single points being excessively attracted by other clusters, thereby preserving the local structural information of the original cluster during group contraction and significantly enhancing the clarity and stability of inter-cluster boundaries.

In the following section, we provide a detailed description of how collaboratively moving groups are constructed.

\textbf{Local Density Estimation.}  
For a dataset \( \mathcal{X} = \{x_1, x_2, \dots, x_N\} \), we first compute the \emph{local density} \(\rho_i\) for each data point based on a truncation radius \( r \), and then obtain the \emph{density lower bound} $\rho^{\dagger}$ via the global average. Inspired by the DPC algorithm \cite{DPC2014}, our experiments indicate that selecting the distance to the nearest 1.5\% of neighbors for each point as the truncation radius can ensure accurate density estimation while effectively suppressing the influence of local outliers, yielding stable and high-precision clustering results across different datasets. Therefore, this empirical value is adopted as the basis for truncation radius computation in this study. Relevant definitions are as follows:

\begin{definition}[Truncation Radius]  
For a dataset $\mathcal{X}$, the truncation radius $r$ is defined as:
\[
r = \frac{1}{N} \sum_{i=1}^{N} r_i,
\]
where \( r_i \) denotes the distance from point $x_i$ to its nearest 1.5\% neighbor, and $N$ is the total number of data points in the dataset.
\end{definition}

\begin{definition}[Local Density]\label{def:pt_ds}  
For any data point \(x_i \in \mathcal{X}\), its local density $\rho_i$ is defined as:
\[
\rho_i = \left| \{ x_l \in \mathcal{X} - \{x_i\} \mid \|x_l - x_i\|_2 \leq r \} \right|,
\]
i.e., the number of neighboring points within radius \( r \) (excluding itself).
\end{definition}

Fig. \ref{fig:r_i} illustrates the visualization of the truncation radius and local neighborhood. As shown, when using $r_i$ as the radius, there are four neighbors within this range (marked as red points in the figure), resulting in $\rho_i=4$.

\begin{figure}[hbpt]
    \includegraphics[width=0.5\textwidth]{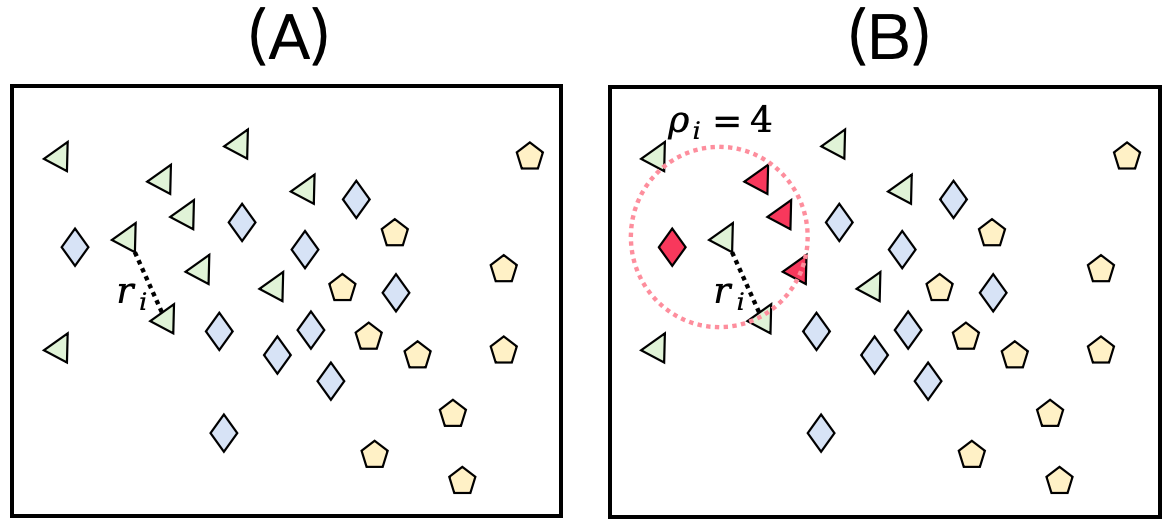}
    \caption{An example of local density.}
    \label{fig:r_i}
\end{figure}

\begin{definition}[Density Lower Bound]\label{def:pt_rg}  
For a dataset $\mathcal{X}$, the density lower bound $\rho^{\dagger}$ is defined as:
\[
\rho^{\dagger} = \frac{1}{N} \sum_{i=1}^{N} \rho_i.
\]
\end{definition}

Using the local density $\rho_i$ defined in Definition~\ref{def:pt_ds}, we can accurately characterize the local structural features of each data point; based on the computed $\rho^{\dagger}$ from Definition~\ref{def:pt_rg}, data points can be distinguished as high-density or low-density points.

\textbf{Collaboratively Moving Group Mechanism.}  
As mentioned earlier, low-density points are often susceptible to strong attraction from nearby inter-cluster points, which can disrupt cluster structures and blur boundaries. Therefore, the collaboratively moving groups in this work are specifically constructed around low-density points. Member points are treated as the basic units for gravitational interactions, with low-density points and their neighboring points collectively forming a \textbf{collaboratively moving group} that serves as the fundamental movement unit. We first define the set of low-density points (Definition~\ref{def:GOP_st}).

\begin{definition}[Low-Density Point Set]\label{def:GOP_st}  
For a dataset \(\mathcal{X}\), the low-density point set $\mathcal{L}$ is defined as:
\[
\mathcal{L} = \{x_j \mid \rho_j < \rho^{\dagger} \ \text{and} \ \rho_j > 0\}.
\]
Here, \(x_j\) represents a data point whose local density is below the density threshold $\rho^{\dagger}$.
\end{definition}

Building on this, we further define collaboratively moving groups. In our framework, each collaboratively moving group serves as a basic movement unit for gravitational contraction, and the same neighboring point may belong to multiple groups simultaneously. To ensure reasonable group assignment and contraction, a member point is assigned to group $G_j$ if the majority of its $k$ nearest neighbors belong to $G_j$ rather than $G_p$, and participates as a member of $G_j$ in subsequent collaborative movement. Example~\ref{exp:shared_pts} details the assignment process for member points that belong to multiple groups.

\begin{example}
Fig. \ref{fig:shared_points} illustrates the specific decision process for member point assignment. Suppose the dataset contains 30 points, and in the initial group formation result, Group $j$ and Group $p$ share the same member point, as shown in Fig.~\ref{fig:shared_points}(A). To determine the final assignment of this point, we take it as the center and search for its $k$ nearest neighbors (here $k = 5$), as shown in Fig.~\ref{fig:shared_points}(B), and count how many of these neighbors belong to Group $j$ and Group $p$, respectively. As seen in Fig.~\ref{fig:shared_points}(B), the majority of neighbors belong to Group $p$, so the shared point is ultimately assigned to Group $p$, as shown in Fig.~\ref{fig:shared_points}(C).
\label{exp:shared_pts}
\end{example}

\begin{figure}[htbp]
    \includegraphics[width=0.49\textwidth]{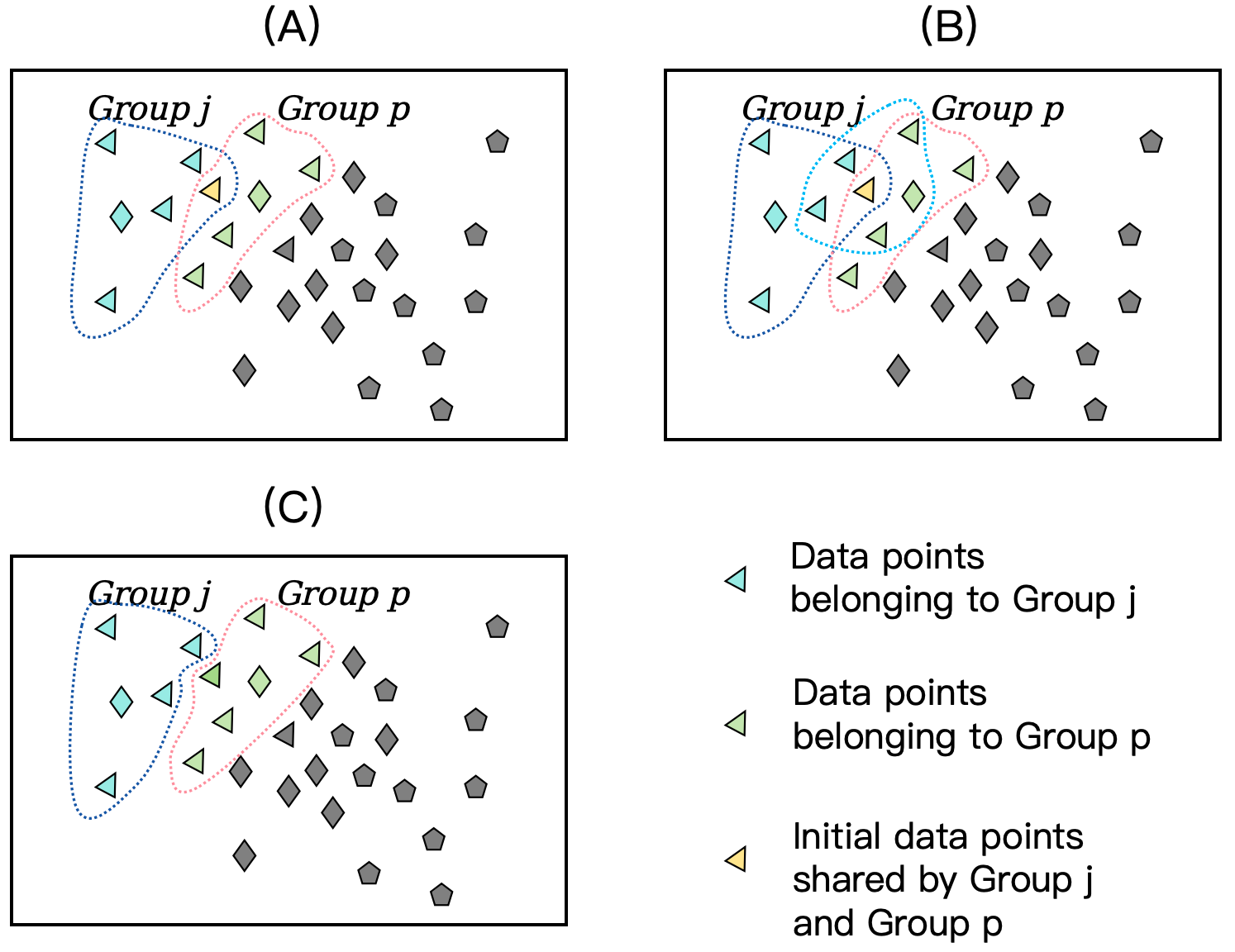}
    \caption{An example of processing shared points.}
    \label{fig:shared_points}
\end{figure}

Meanwhile, the group assignment is performed only once during the initialization stage to ensure that the groups maintain their initial structure during subsequent contraction movements, aligning with the algorithm's property of preserving the local structure of the original clusters. The collaboratively moving group is defined as follows:

\begin{definition}[Collaboratively Moving Group]\label{def:tg_gp}
Let the set of low-density points be $\mathcal{L} = \{x_1, x_2, \dots, x_m\}$. For any low-density point $x_j \in \mathcal{L}$, its collaboratively moving group $G_j$ is defined as:
\[
G_j = \{x_j\} \cup \{y_i \mid y_i \in \mathcal{N}_k(x_j),\ 
n_j(y_i) > n_p(y_i),\ \forall p \neq j \},
\]
where $\mathcal{N}_k(x_j)$ denotes the set of $k$ nearest neighbors of $x_j$, $n_j(y_i)$ is the number of $y_i$'s $k$ neighbors that belong to group $G_j$, and $n_p(y_i)$ is the number of $y_i$'s $k$ neighbors that belong to group $G_p$.
\end{definition}

In summary, the GFP establishes collaboratively moving groups centered on low-density points, preserving local structural information in cluster boundary regions and laying a solid foundation for the subsequent GOP.

\subsection{GOP}

\textbf{Main Idea.}  
In traditional gravitational models, position updates are typically based on \emph{point-to-point} resultant force calculations; however, the collaboratively moving mechanism proposed in this study treats \emph{groups as the basic movement units}, computing the group resultant force by integrating the forces acting on all member points and updating the group's overall position accordingly. Each member point participates as a basic unit transmitting forces, while the group acts as the executor of contraction and migration, achieving a transition from ``individual-driven'' to ``group-coordinated'' motion. Consequently, the traditional single-point force update framework is no longer directly applicable. To address this challenge, we innovatively design a \textbf{Group Gravitational Response Mechanism}, introducing member forces, group forces, and group contraction vectors to regulate inter-group interactions, thereby achieving a more stable, efficient, and interpretable gravitational contraction process.

\textbf{Group Gravitational Response Mechanism.}  
The Group Gravitational Response Mechanism is implemented through the following definition to update the positions of collaboratively moving groups.

\begin{definition}[Gravitational Response Mechanism]
For any two data points \( (x_m, x_n) \in \mathcal{X} \), the gravitational response vector \( \vec{F}_{mn} \) is defined as:
\[
\vec{F}_{mn} = \lambda \cdot \frac{d_{mz}}{d_{mn}}, \quad d_{mn} = \|x_m - x_n\|_2,
\]
where:
\begin{itemize}
    \item $\lambda$ is the step size coefficient;
    \item \(x_m, x_n\) are any two points in the data space;
    \item \(d_{mn}\) is the Euclidean distance between points \(x_m\) and \(x_n\);
    \item \(d_{mz}\) is the Euclidean distance between \(x_m\) and its nearest neighbor \(x_z\);
    \item \( \vec{F}_{mn} \) is the gravitational response vector exerted by \(x_n\) on \(x_m\), pointing from \(x_n\) to \(x_m\).
\end{itemize}
\end{definition}

The gravitational response vector \( \vec{F}_{mn} \) between data points is determined by the Euclidean distance \( d_{mn} \) and the step size coefficient \( \lambda \). This calculation accurately captures the spatial relationship between low-density points and their neighbors: the closer the points, the stronger the force; conversely, the farther apart, the weaker the force, enabling fine-grained control over the influence on local structures.

In this study, two types of forces are defined: the first is the \textbf{Member Force} shown in Definition~\ref{def:mb_f}, and the second is the \textbf{Group Force} shown in Definition~\ref{def:gp_fc}.

\begin{definition}[Member Force]
    For any member point \( x_i \in G_j \) in a collaboratively moving group \( G_j \), we first select its \( k \) nearest neighbors and exclude points inside the group, retaining only the external neighborhood set:
    \[
    \mathcal{N}_k^{\text{out}}(x_i) = \{ x_n \mid x_n \in \mathcal{N}_k(x_i),\ x_n \notin G_j \}.
    \]
    The member force on this point is then defined as:
    \[
    \vec{F}_i = \sum_{x_n \in \mathcal{N}_k^{\text{out}}(x_i)} \vec{F}_{ni}.
    \]
    Here, \(\vec{F}_i\) denotes the total gravitational force exerted on the member point \(x_i\).
    \label{def:mb_f}
\end{definition}

\begin{figure}[htbp]
    \centering
    \includegraphics[width=0.27\textwidth]{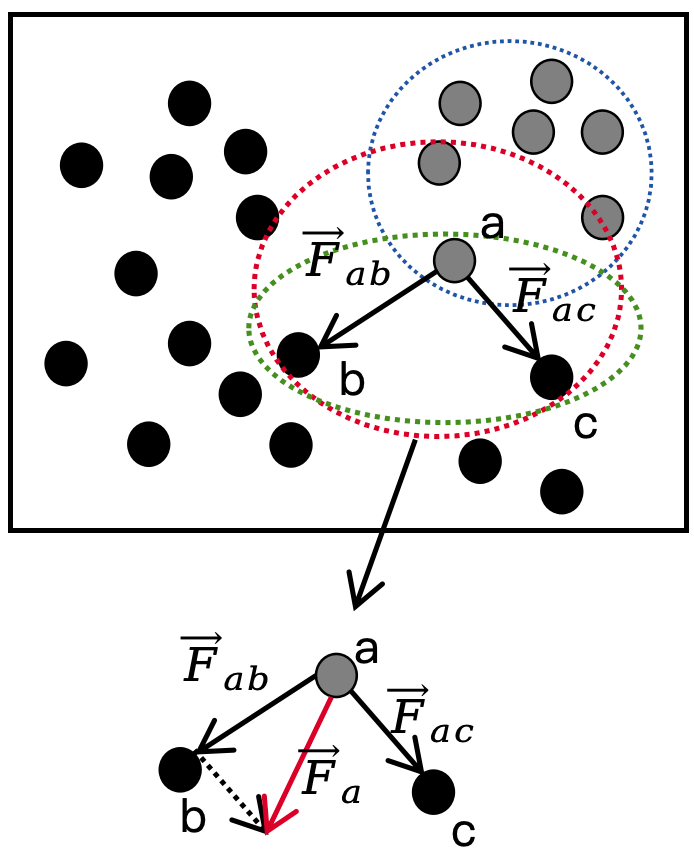}
    \caption{Illustration of the gravitational forces acting on a member point}
    \label{fig:force_fi}
\end{figure}

The design of the member force effectively captures the combined influence of external points on the member points within the group while avoiding redundant interactions among internal members. As illustrated in Fig.~\ref{fig:force_fi}, the diagram visually demonstrates the external forces acting on a single member point. Points inside the blue circle represent members of the same group, points inside the red circle represent neighboring points with potential gravitational influence on the current member, and points inside the green circle denote the actual external points exerting effective force on the member. This design ensures that the forces acting on group members originate solely from external points, thereby avoiding redundant internal interactions, maintaining group structure stability, and improving overall force computation efficiency.

\begin{definition}[Group Force]\label{def:gp_fc}  
    For each collaboratively moving group \(G_j\), the group force is defined as:
    \[
    \vec{F}_{G_j} = \frac{1}{|G_j|} \sum_{x_i \in G_j} \vec{F}_{i},
    \]
    where \( \vec{F}_{i} \) is the member force of point \(x_i\), and \(|G_j|\) denotes the number of members in group \(G_j\).  
\end{definition}

Based on the computed group force, the group contraction vector \( \Delta x_j \) is defined as follows:

\begin{definition}[Group Contraction Vector]\label{def:gt_vc}  
For group \(G_j\), let \(x_j \in G_j\) be a data point:
\[
\Delta x_j = \vec{F}_{G_j}.
\] 
\end{definition}

The computed group contraction vector \( \Delta x_j \) is applied to all member points in the group \(G_j\). Specifically, after each iteration, the position of member point \(x_j\) at iteration \(t+1\) is updated as follows:

\begin{definition}[Group Position Update]\label{def:gt_pos_update}  
For group \(G_j\), let \(x_j \in G_j\), the position of \(x_j\) at the \(t+1\)-th contraction iteration is:
\[
x_{j}^{t} = x_{j}^{t-1} + \Delta x_{j}^{t-1},
\] 
\end{definition}
\begin{theorem}[Boundary Clarity]\label{thm:boundary_clarity}
During the contraction process of collaboratively moving groups, let $C_1$ and $C_2$ be two distinct clusters, and define the minimum distance between clusters as
\[
d_{\min}(C_1,C_2) = \min_{x \in B_1, y \in B_2} \|x-y\|_2,
\]
where $B_1$ and $B_2$ are the sets of boundary points of $C_1$ and $C_2$, respectively. After the $t$-th group contraction iteration, it holds that
\[
d_{\min}^{(t+1)}(C_1,C_2) \ge d_{\min}^{(t)}(C_1,C_2),
\]
i.e., the relative boundary clarity between clusters does not decrease.
\end{theorem}

\begin{proof}
\renewcommand{\qedsymbol}{}

$\because$ Let $x_i \in C_1$ and $y_j \in C_2$ be boundary points belonging to groups $G_i$ and $G_j$, respectively. Due to the density distribution, the majority of $x_i$'s neighbors lie closer to the cluster center of $C_1$ than to $C_2$:
\[
|\mathcal{N}_k(x_i) \cap C_1| > |\mathcal{N}_k(x_i) \cap C_2|.
\]

\noindent$\because$ The group force $\vec{F}_{G_i}$ is defined as the sum of external forces acting on members of $G_i$:
\[
\vec{F}_{G_i} = \sum_{x \in G_i} \sum_{y \in \mathcal{N}_k^{\text{out}}(x)} \vec{F}_{yx}.
\]

\noindent$\therefore$ The resulting force is dominated by intra-cluster attractions and the effect of points from other clusters is diluted.

\noindent$\therefore$ The group moves collectively towards its cluster center:
\[
\begin{gathered}
x_i^{(t+1)} = x_i^{(t)} + \vec{F}_{G_i} \\
\Rightarrow \quad 
\|x_i^{(t+1)} - \text{center}(C_1)\|_2 
< 
\|x_i^{(t)} - \text{center}(C_1)\|_2
\end{gathered}
\]

\noindent Similarly,
\[
\begin{gathered}
y_j^{(t+1)} = y_j^{(t)} + \vec{F}_{G_j} \\
\Rightarrow \quad 
\|y_j^{(t+1)} - \text{center}(C_2)\|_2 
< 
\|y_j^{(t)} - \text{center}(C_2)\|_2
\end{gathered}
\]

\noindent Consequently, the distance between boundary points increases or remains the same:
\[
\|x_i^{(t+1)} - y_j^{(t+1)}\|_2 \ge \|x_i^{(t)} - y_j^{(t)}\|_2.
\]

\noindent$\therefore$ 
\[
d_{\min}^{(t+1)}(C_1,C_2) \ge d_{\min}^{(t)}(C_1,C_2),
\]
which proves the relative boundary clarity is preserved or improved.
\end{proof}

\begin{theorem}[Cluster Integrity]\label{thm:cluster_integrity}
Let a cluster $C$ consist of multiple collaboratively moving groups $\{G_1, G_2, \dots, G_M\}$, and let $\forall x_i, x_j \in C$. After the $t$-th contraction of group $G_m$, points within the same cluster move closer to each other, preserving the cluster structure:
\[
\|x_i^{(t+1)} - x_j^{(t+1)}\|_2 \le \|x_i^{(t)} - x_j^{(t)}\|_2.
\]
\end{theorem}

\begin{proof}
\renewcommand{\qedsymbol}{}

$\because$ For $\forall x_i, x_j \in \forall G_m \subset C$, the movement of $x_i, x_j$ is determined by $\vec{F}_{G_m}$:
\[
\begin{aligned}
x_i^{(t+1)} &= x_i^{(t)} + \Delta x_i^{(t)} \\
&= x_i^{(t)} + \vec{F}_{G_m} \\[2mm]
x_j^{(t+1)} &= x_j^{(t)} + \Delta x_j^{(t)} \\
&= x_j^{(t)} + \vec{F}_{G_m},
\end{aligned}
\]

$\therefore$ $x_i$ and $x_j$ share the same moving direction and magnitude, maintaining their relative positions:
\[
x_i^{(t+1)} - x_j^{(t+1)} = x_i^{(t)} - x_j^{(t)}
\]

$\therefore$ 
\[
\|x_i^{(t+1)} - x_j^{(t+1)}\|_2 = \|x_i^{(t)} - x_j^{(t)}\|_2.
\]

$\because$ For $\forall x_i \in G_i \subset C, \forall x_j \in G_j \subset C$, we have
\[
\begin{aligned}
\|x_i^{(t+1)} - x_j^{(t+1)}\|_2
&= x_i^{(t+1)} - x_j^{(t+1)}\\
&= x_i^{(t)} + \Delta x_i^{(t)} - (x_j^{(t)} + \Delta x_j^{(t)}) \\
&= x_i^{(t)} + \vec{F}_{G_i} - (x_j^{(t)} + \vec{F}_{G_j}) \\
&= (x_i^{(t)} - x_j^{(t)}) \\
&- \Big(\sum_{x_i^{(t)} \in G_i} 
        \sum_{y \in \mathcal{N}_m^{\text{out}}(x_i^{(t)})} 
        \vec{F}_{y(x_i^{(t)})} \\
        &+ \sum_{x_j^{(t)} \in G_j} 
        \sum_{y \in \mathcal{N}_m^{\text{out}}(x_j^{(t)})} 
        \vec{F}_{y(x_j^{(t)})}\Big)\\
&= \|x_i^{(t)} - x_j^{(t)}\|_2\\
&- \Big(\sum_{x_i^{(t)} \in G_i} 
        \sum_{y \in \mathcal{N}_m^{\text{out}}(x_i^{(t)})} 
        \vec{F}_{y(x_i^{(t)})} \\
        &+ \sum_{x_j^{(t)} \in G_j} 
        \sum_{y \in \mathcal{N}_m^{\text{out}}(x_j^{(t)})} 
        \vec{F}_{y(x_j^{(t)})}\Big)
\end{aligned}
\]

$\because$ From Theorem~\ref{thm:boundary_clarity}, for different groups $G_i, G_j$ within the same cluster $C$, their group forces point towards the cluster density center $\text{center}(C)$, with an angle $\theta < 90^\circ$, i.e.,
\[
\vec{F}_{G_i} \cdot \vec{F}_{G_j} > 0.
\]

$\because$
\[
\|\vec{F}_{G_i} + \vec{F}_{G_j}\| \ge \max(\|\vec{F}_{G_i}\|, \|\vec{F}_{G_j}\|) > 0,
\]

$\therefore$ 
\[
\vec{F}_{G_i} + \vec{F}_{G_j} > 0.
\]

$\therefore$ 
\[
\sum_{x_i^{(t)} \in G_i} 
\sum_{y \in \mathcal{N}_m^{\text{out}}(x_i^{(t)})} 
\vec{F}_{y(x_i^{(t)})} + \sum_{x_j^{(t)} \in G_j} 
\sum_{y \in \mathcal{N}_m^{\text{out}}(x_j^{(t)})} 
\vec{F}_{y(x_j^{(t)})} > 0
\]

$\therefore$ 
\[
\begin{aligned}
\|x_i^{(t+1)} - x_j^{(t+1)}\|_2
&= \|x_i^{(t)} - x_j^{(t)}\|_2 \\
&- \Big(\sum_{x_i^{(t)} \in G_i} 
        \sum_{y \in \mathcal{N}_m^{\text{out}}(x_i^{(t)})} 
        \vec{F}_{y(x_i^{(t)})} \\
        &+ \sum_{x_j^{(t)} \in G_j} 
        \sum_{y \in \mathcal{N}_m^{\text{out}}(x_j^{(t)})} 
        \vec{F}_{y(x_j^{(t)})}\Big)\\
&< \|x_i^{(t)} - x_j^{(t)}\|_2
\end{aligned}
\]

$\therefore$ 
\[
\|x_i^{(t+1)} - x_j^{(t+1)}\|_2 \le \|x_i^{(t)} - x_j^{(t)}\|_2.
\]
\end{proof}

The direction and magnitude of the group force essentially correspond to the gradient of the local density field, causing the group to naturally move along the density gradient during migration, driving low-density points in the boundary region toward the cluster center. This gravitational contraction approach comprehensively reflects the spatial relationships between the points within the group and multiple points outside the group, while performing robust position updates that preserve local structural information. Overall, the GOP updates the positions of collaboratively moving groups via member forces and group forces, effectively mitigating the impact of excessively close inter-cluster points on single-point forces and maintaining the spatial structural information of the data points.

\subsection{Time Complexity Analysis}

In the GCAO algorithm, the overall procedure mainly consists of three stages: neighborhood search and density estimation, group construction and assignment, and group gravitational optimization. The time complexity of each stage is analyzed as follows.

\textbf{(1) Neighborhood Search and Local Density Estimation.}  
In the GFP stage, the algorithm first computes the $k$ nearest neighbors of each sample using a Ball-Tree-based $k$-nearest neighbor search (\texttt{NearestNeighbors}). The time complexity of this process is $O(N \log N)$, where $N$ is the number of samples. Subsequently, the local density of each point is computed by counting the number of neighbors within a certain distance threshold, which has a complexity of $O(Nk)$. Therefore, the total time complexity of this stage can be expressed as:
\[
O(N \log N + Nk).
\]
\textbf{(2) Group Construction and Connected Component Assignment.}  
The algorithm constructs an adjacency list for the low-density point set and identifies connected groups via Depth-First Search (DFS). Let the proportion of low-density points be $\alpha$, so this subset contains $\alpha N$ points. Constructing the adjacency list requires $O(\alpha Nk)$, and DFS traversing each node and edge has a complexity of $O(\alpha N + \alpha Nk)$, thus the overall complexity of this stage is approximately:
\[
O(\alpha Nk).
\]

\textbf{(3) Group Gravitational Optimization and Parallel Collaborative Contraction.}  
In the GOP stage, the algorithm computes the average gravitational vector for each group in parallel. Let the number of groups be $G$, with an average group size $S = \frac{\alpha N}{G}$. Each group iterates over $S$ samples and their $k$ external neighbors, so the complexity for a single group is $O(Sk)$, and the total complexity for all groups is $O(\alpha Nk)$.  
Considering internal thread parallelization (\texttt{internal\_threads = T}) and assuming near-ideal efficiency, the effective complexity of this stage is:
\[
O\left(\frac{\alpha Nk}{T}\right).
\]

\textbf{(4) Parameter Search and Multi-Process Outer Parallelism.}  
In the outer parameter search stage, the algorithm executes contraction and clustering in parallel over $P$ parameter sets, each performing $L$ contraction iterations. Using $W$ processes, the total complexity can be approximated as:
\[
O\left(\frac{PL}{W} \times (N \log N + Nk)\right).
\]
Since in experiments $L$ and $P$ are relatively small (typically $L < 10$, $P < 200$), the overall complexity is dominated by the $O(N \log N)$ and $O(Nk)$ terms.

\textbf{(5) Summary of Overall Complexity.}  
Combining the above stages, the total time complexity for a single complete execution of the GCAO algorithm can be expressed as:
\[
O(N \log N + Nk + \alpha Nk/T).
\]
When $\alpha \ll 1$ (low-density points occupy only a small fraction), $k \ll N$, and the number of parallel threads $T$ is large, GCAO can achieve near-linear scalability while maintaining good convergence performance.  
Therefore, the overall complexity of GCAO for large-scale datasets approaches $O(N \log N)$, demonstrating excellent computational efficiency and scalability in practice.

\section{EXPERIMENTS}
\subsection{Experimental Setting}
In this subsection, we describe the datasets, evaluation metrics, and baseline algorithms.

\subsubsection{Datasets}
To validate the effectiveness of the proposed method, we selected several real-world large-scale datasets, including Human Activity Recognition (HAR)\cite{HAR2013}, Dry Bean\cite{DryBean2020}, Letter\cite{Letter1991}, and CIFAR-10\cite{CIFAR102009}. These datasets are characterized by high dimensionality and large sample sizes, making them suitable for evaluating clustering and classification performance. Table \ref{tab:dataset} summarizes the basic information of each dataset, including the number of samples, feature dimensions, and the number of classes.

\begin{table}[h]
    \centering
    \caption{Basic information of the datasets used}
    \label{tab:dataset}
    \begin{tabular}{lccc}
        \hline
        Dataset & Samples & Features & Classes \\
        \hline
        HAR & 10299 & 561 & 6\\
        Dry Bean & 13611 & 16 & 7\\
        Letter & 20000 & 16 & 16\\
        CIFAR-10 & 60000 & 512 & 10\\
        \hline
    \end{tabular}
\end{table}

\begin{itemize}
    \item \textbf{HAR}: This dataset records environmental sensor data of volunteers performing daily activities at home. The data are continuously collected, reflecting motion characteristics in real-life settings.
    \item \textbf{Dry Bean}: Contains 7 different types of dry beans. Using a computer vision system, 13,611 beans were photographed, and 16 features (including 12 dimensional features and 4 shape features) were extracted for classification.
    \item \textbf{Letter}: Composed of images of 26 English letters generated from 20 different fonts. After applying distortion, 20,000 independent samples were created, each with 16 numerical features. The training and test sets contain 16,000 and 4,000 samples, respectively.
    \item \textbf{CIFAR-10}: A subset of the original Tiny Images dataset, containing color images from 10 classes, totaling 60,000 images. It is widely used for image classification and generative model evaluation.
\end{itemize}

\subsubsection{Evaluation Metrics}

To comprehensively evaluate the quality of clustering results, we selected several commonly used external and internal metrics, including Normalized Mutual Information (NMI), Adjusted Rand Index (ARI), Homogeneity, and Accuracy (ACC). Among them, NMI measures the consistency between clustering results and ground truth labels, calculated as:

\begin{equation}
\text{NMI}(U, V) =
\frac{
\displaystyle \sum_{i=1}^{|U|} \sum_{j=1}^{|V|} P(i, j) \log \frac{P(i, j)}{P(i) P'(j)}
}
{
\displaystyle \sqrt{
\Big( \sum_{i=1}^{|U|} P(i) \log P(i) \Big)
\Big( \sum_{j=1}^{|V|} P'(j) \log P'(j) \Big)
}
}
\end{equation}

where $U$ denotes the set of ground truth labels, $V$ denotes the clustering result set, $P(i) = \frac{|U_i|}{N}$, $P'(j) = \frac{|V_j|}{N}$, $P(i, j) = \frac{|U_i \cap V_j|}{N}$, and $N$ is the total number of samples. NMI ranges in $[0,1]$, with higher values indicating better consistency between clustering results and ground truth labels. In our experiments, the proposed method achieved high NMI scores across all datasets.

The Adjusted Rand Index (ARI) measures the pairwise agreement between clustering results and ground truth labels, defined as:
\begin{equation}
\text{ARI} = \frac{\text{RI} - \mathbb{E}[\text{RI}]}{\max(\text{RI}) - \mathbb{E}[\text{RI}]}
\end{equation}
where RI (Rand Index) represents the proportion of correctly and incorrectly clustered pairs:
\begin{equation}
\text{RI} = \frac{TP + TN}{TP + FP + TN + FN}
\end{equation}
Here, $TP$ denotes the number of true positive pairs correctly clustered together, $TN$ denotes the number of true negative pairs correctly separated, and $FP$ and $FN$ represent the number of false positive and false negative pairs, respectively.

\subsubsection{Evaluation Metrics (continued)}

Homogeneity measures whether the samples within each cluster belong to the same class, ranging from $[0,1]$, with higher values indicating purer clusters. In our experiments, we obtained $\text{Homogeneity} = 0.7385$. Accuracy (ACC) evaluates the consistency between clustering results and ground truth labels via the best matching strategy, with an experimental result of $\text{ACC} = 0.8617$. The Silhouette Score assesses the compactness and separation of clusters, ranging from $[-1,1]$.

In summary, NMI, ARI, Homogeneity, ACC, and Silhouette Score collectively reflect the performance of the proposed method in terms of clustering accuracy, intra-cluster purity, and inter-cluster separation from multiple perspectives.

\subsubsection{Baseline Algorithms}

This work focuses on gravity-based clustering methods. To comprehensively evaluate the effectiveness of the proposed method, we selected several clustering algorithms that are also based on gravitational mechanisms, as well as several classical fast density-peak clustering variants as baselines. The comparison algorithms include: HIBOG\cite{HIBOG2021}, HIAC\cite{HIAC2023}, HIACSP\cite{HIACSP2025}, DPCG\cite{DPCG2018}, GDPC\cite{GDPC2019}, Fast-LDP-MST\cite{FastLDPMST2023}, Ultra DPC\cite{UltraDPC2024}, and R-MDPC\cite{RMDPC2024}. These algorithms leverage gravitational interactions, local neighborhood attraction mechanisms, and shortest-path-induced strategies, demonstrating strong expressiveness and adaptability in complex or large-scale data distributions.

\begin{table*}[!t]
\centering
\caption{Performance comparison of clustering algorithms on different datasets (NMI, ARI, Homogeneity, ACC)}
\label{tab:baseline_comparison}
\resizebox{\textwidth}{!}{
\begin{tabular}{l|cccc|cccc|cccc|cccc}
\hline
& \multicolumn{4}{c|}{HAR} 
& \multicolumn{4}{c|}{Dry Bean} 
& \multicolumn{4}{c|}{Letter} 
& \multicolumn{4}{c}{CIFAR-10} \\
\cline{2-5} \cline{6-9} \cline{10-13} \cline{14-17}
Algorithm & NMI & ARI & Hom & ACC 
& NMI & ARI & Hom & ACC 
& NMI & ARI & Hom & ACC 
& NMI & ARI & Hom & ACC \\
\hline
Newtonian & 0.61 & 0.85 & 0.60 & 0.62 & 0.73 & \underline{0.90} & 0.73 & \underline{0.86} & 0.09 & \underline{0.74} & 0.11 & 0.14 & 0.07 & 0.03 & 0.07 & 0.19 \\
Herd & 0.66 & 0.86 & 0.63 & 0.58 & 0.71 & 0.89 & 0.70 & 0.74 & 0.09 & 0.73 & 0.11 & 0.14 & 0.07 & 0.03 & 0.07 & 0.19 \\
SBCA & 0.60 & 0.84 & 0.57 & 0.59 & 0.10 & 0.66 & 0.09 & 0.24 & 0.08 & 0.71 & 0.10 & \underline{0.15} & 0.07 & 0.01 & 0.05 & 0.15 \\
HIBOG & \underline{0.68} & \underline{0.87} & \underline{0.65} & \underline{0.66} & 0.70 & 0.88 & 0.69 & 0.73 & 0.09 & 0.73 & 0.11 & 0.14 & 0.08 & 0.10 & 0.07 & 0.21 \\
HIAC & 0.61 & 0.83 & 0.61 & 0.53 & 0.73 & \underline{0.90} & 0.73 & 0.85 & 0.09 & \underline{0.74} & 0.12 & 0.14 & 0.10 & 0.10 & 0.07 & 0.19 \\
HIACSP & 0.61 & 0.82 & 0.61 & 0.53 & 0.40 & 0.76 & 0.35 & 0.54 & 0.08 & 0.72 & 0.10 & 0.14 & 0.11 & 0.15 & 0.14 & 0.23 \\
DPCG & 0.26 & 0.16 & 0.22 & 0.39 & 0.48 & 0.27 & 0.40 & 0.47 & 0.18 & 0.03 & 0.14 & 0.09 & 0.00 & 0.00 & 0.00 & 0.00 \\
GDPC & 0.51 & 0.32 & 0.45 & 0.60 & \textbf{0.82} & 0.67 & \textbf{0.81} & 0.85 & \textbf{0.27} & 0.08 & \underline{0.23} & 0.12 & \underline{0.33} & 0.14 & \underline{0.24} & \textbf{0.39} \\
Fast-LDP-MST & 0.26 & 0.29 & 0.26 & 0.37 & 0.42 & 0.25 & 0.42 & 0.36 & \underline{0.22} & 0.40 & 0.22 & 0.11 & 0.10 & \underline{0.18} & 0.10 &  0.20\\
Ultra DPC & 0.50 & 0.33 & 0.50 & 0.58 & 0.42 & 0.24 & 0.39 & 0.37 & \textbf{0.27} & 0.08 & \textbf{0.26} & 0.08 & \textbf{0.41} & 0.17 & \textbf{0.34} & 0.31 \\
R-MDPC & 0.10 & 0.57 & 0.12 & 0.43 & 0.38 & 0.31 & 0.38 & 0.37 & 0.21 & 0.20 & 0.21 & 0.11 & 0.06 & 0.10 & 0.06 &  0.14\\
\textbf{GCAO(Proposed Method)} & \textbf{0.69} & \textbf{0.88} & \textbf{0.66} & \textbf{0.69} & \underline{0.74} & \textbf{0.92} & \underline{0.74} & \textbf{0.88} & 0.11 & \textbf{0.77} & 0.14 & \textbf{0.16} & 0.29 & \textbf{0.21} & 0.25 & \underline{0.33} \\
\hline
\end{tabular}
}
\label{tab:baseline_comparision}
\end{table*}

\begin{itemize}
    \item \textbf{Newtonian}: Treats each data point as a particle with ``mass'' and defines inter-point gravitational forces, achieving continuous potential energy reduction and convergence, ultimately forming stable clusters at gravitational equilibrium points.
    
    \item \textbf{Herd}: Introduces collective behavior mechanisms, allowing data points to self-organize and aggregate via ``group movement,'' enhancing coordination and stability in cluster formation.
    
    \item \textbf{SBCA}: Combines swarm intelligence with gravitational attraction, optimizing sample positions through group cooperative contraction to improve clustering convergence speed and noise robustness.

    \item \textbf{HIBOG}: Introduces class-based gravitational forces between each object and its $K$ nearest neighbors, attracting objects towards neighbors to achieve stepwise cluster aggregation.

    \item \textbf{HIAC}: Improves the adjacency strategy of HIBOG by distinguishing valid and invalid neighbors, retaining gravitational effects only among valid neighbors, enhancing boundary recognition and overall clustering accuracy.

    \item \textbf{HIACSP}: Introduces a shortest-path-based distance metric $\delta_{SP}$, encouraging $K$ nearest neighbors to concentrate around the same cluster core, retaining attraction only among intra-cluster objects without requiring an attraction threshold, effectively preventing micro-cluster phenomena and boundary point shifts.

    \item \textbf{DPCG}: Divides the dataset into equidistant grids, treats each grid as an independent object, and then clusters these grid objects using DPC, improving large-scale data processing efficiency.

    \item \textbf{GDPC}: Samples high-density objects in the data space, performs DPC clustering on the sampled objects, and assigns remaining objects to the nearest cluster, achieving efficient clustering.

    \item \textbf{Fast-LDP-MST}: Samples local density peak objects based on $k$ nearest neighbors, performs DPC clustering on these objects, and assigns unsampled objects to their nearest clusters, reducing computational cost.

    \item \textbf{Ultra DPC}: Samples high-density objects via the t-DPE method, clusters sampled objects using DPC, and assigns remaining objects to the corresponding nearest sampled cluster, balancing speed and clustering quality.

    \item \textbf{R-MDPC}: Similar to Fast-LDP-MST, samples local density peak objects based on $k$ nearest neighbors and performs DPC clustering, maintaining efficiency while ensuring clustering accuracy.
\end{itemize}

\subsection{Comparison Experiment}
To comprehensively evaluate the effectiveness and advancement of the proposed GCAO algorithm, we conducted detailed comparison experiments on four publicly available datasets with varying dimensions, scales, and distribution characteristics (HAR, Dry Bean, Letter, CIFAR-10) against 11 mainstream clustering algorithms.

The selected baseline algorithms are divided into two categories:  
1) \textbf{Gravity-based models}: including Newtonian, Herd, SBCA, HIBOG, HIAC, and HIACSP. These methods utilize the gravitational or contraction mechanisms among data points to achieve clustering.  
2) \textbf{Density-peak (DPC) or sampling-based models}: including DPCG, GDPC, Fast-LDP-MST, Ultra DPC, and R-MDPC. These methods improve clustering efficiency and effectiveness by identifying high-density regions or representative sampled points.

All experiments were conducted under a consistent environment, using four standard evaluation metrics: Normalized Mutual Information (NMI), Adjusted Rand Index (ARI), Homogeneity, and Accuracy (ACC). Table \ref{tab:baseline_comparison} presents the detailed performance of all algorithms on each dataset, where \textbf{bold} indicates the best result and \underline{underlined} indicates the second-best result.

\textbf{Overall Performance Analysis.}  
From the overall results in Table \ref{tab:baseline_comparison}, the proposed GCAO algorithm demonstrates highly competitive performance across all four datasets, achieving either the best or second-best results on multiple core metrics.

Specifically, on the high-dimensional \textbf{HAR} dataset, GCAO outperforms all compared algorithms on NMI, ARI, Homogeneity, and ACC, achieving the best performance (NMI=0.69, ARI=0.88, ACC=0.69). On the \textbf{Dry Bean} dataset, GCAO also performs excellently, achieving the highest ARI (0.92) and ACC (0.88), while ranking second in NMI and Homogeneity, just behind GDPC.

These results fully demonstrate the effectiveness of the GCAO algorithm. Compared with traditional methods, the designed GFP and GOP work synergistically. By integrating local density and gravitational mechanisms, data points move and contract cooperatively as groups. This mechanism significantly reduces the risk of incorrect segmentation of points within the same cluster in complex spaces and effectively enhances the clarity of inter-cluster boundaries, leading to superior performance on metrics such as ARI and ACC that emphasize correct pairwise clustering and label assignment accuracy.

\textbf{Comparison with Gravity-based Models.}  
As an improved gravity-based model, GCAO demonstrates its advantages particularly on high-dimensional data. Experimental results show that traditional gravity models (such as Newtonian, HIBOG, HIAC) can maintain reasonable performance on medium- and low-dimensional datasets like HAR (e.g., HIBOG achieves an ARI of 0.87 on HAR), but their performance drops drastically when extended to high-dimensional, complex-structured visual data such as \textbf{CIFAR-10} (ARI of HIBOG plummets to 0.10, and Newtonian only reaches 0.03).

In contrast, GCAO achieves an ARI of 0.21 and an ACC of 0.33 (second-best) on CIFAR-10, significantly outperforming all other gravity-based models. This indicates that the group cooperative optimization mechanism in GCAO replaces the traditional single-point contraction pattern, allowing the algorithm to better preserve intra-cluster topological structures during optimization, avoid gravity failure due to high-dimensional sparsity, and prevent boundary point drift, thus demonstrating stronger robustness.

\textbf{Comparison with DPC and Sampling-based Models.}  
Compared with DPC and sampling-based algorithms (such as GDPC, Ultra DPC), GCAO exhibits different performance emphases. GDPC and Ultra DPC perform well on NMI and Homogeneity metrics for \textbf{Letter} and \textbf{CIFAR-10} datasets, reflecting their reliance on sampling high-density objects to identify macroscopic cluster core regions effectively.

\begin{figure*}[htbp]
    \centering
    \includegraphics[width=1\textwidth]{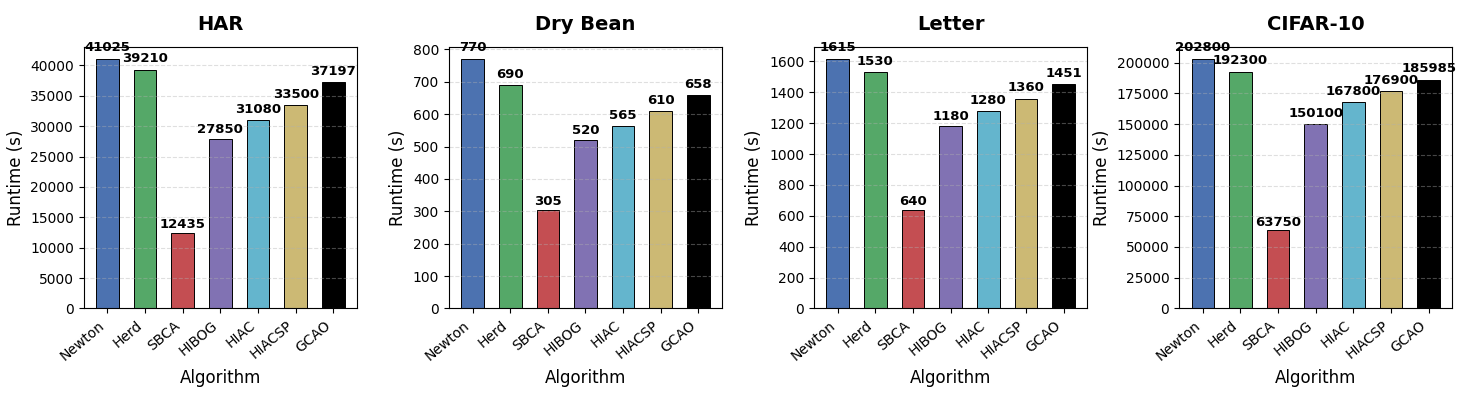}
    \caption{Visualization of runtime}
    \label{fig:runtime_comparison}
\end{figure*}

However, GCAO surpasses these methods in ARI and ACC metrics on the same datasets. For example, on the Letter dataset with 16 classes, GCAO achieves an ARI of 0.77, much higher than Ultra DPC (0.08) and GDPC (0.08); on CIFAR-10, GCAO's ARI (0.21) also exceeds Ultra DPC (0.17).

This comparison clearly demonstrates that although GCAO may be slightly inferior in macroscopic cluster structure recognition (NMI/Homogeneity) compared to optimal sampling algorithms, it provides more precise assignments at the micro-level of sample pairs (ARI) and final label assignment (ACC). This further validates the effectiveness of GCAO's GOP in handling complex cluster boundaries and noise, avoiding partition bias that may arise in sampling-based methods due to inappropriate selection of representative points.

\textbf{Runtime Efficiency Analysis.}  
As shown in Fig. \ref{fig:runtime_comparison}, we compare the runtime performance of GCAO, Newton, Herd, SBCA, HIBOG, HIAC, and HIACSP on four representative datasets (HAR, Dry Bean, Letter, CIFAR-10). Overall, GCAO demonstrates superior computational efficiency and scalability while maintaining high clustering accuracy. Specifically, for the CIFAR-10 dataset, which has both a large sample size and high feature dimensionality, all algorithms exhibit relatively long runtimes, with Newton and Herd incurring the highest computational costs. Nevertheless, GCAO maintains fast convergence even under high-dimensional and complex distributions. On the HAR dataset, GCAO's runtime is significantly lower than traditional point-level gravity-based algorithms (e.g., HIAC, HIBOG), validating the advantage of the group cooperative contraction mechanism in reducing redundant computations and optimizing convergence paths.  

For smaller datasets such as Dry Bean and Letter, GCAO's runtimes are 658 s and 1451 s, respectively, ranking in the mid-to-low range among all methods. This demonstrates its computational stability and efficient resource utilization even for lightweight tasks. Overall, GCAO's group-driven cooperative optimization mechanism significantly enhances clustering accuracy and boundary consistency in complex scenarios while effectively controlling algorithmic complexity, achieving balanced performance across different data scales and dimensionalities.

\textbf{Experimental Summary.}  
In summary, compared with 11 representative algorithms, GCAO achieves average improvements of approximately \textbf{37.13\%}, \textbf{52.08\%}, \textbf{44.98\%}, and \textbf{38.81\%} on NMI, ARI, Homogeneity, and ACC, respectively, while maintaining excellent runtime efficiency and scalability. The comparative experiments fully validate that GCAO achieves a well-balanced trade-off between clustering quality, stability, and computational efficiency. Whether on high-dimensional datasets (HAR, CIFAR-10) or complex distribution datasets (Dry Bean, Letter), GCAO exhibits outstanding robustness and adaptability. The experimental results systematically confirm the effectiveness and advancement of the proposed integration of local density estimation and group gravitational optimization, providing a more general and stable solution for gravity-driven clustering algorithms.

\begin{figure*}[htbp]
    \centering
    \includegraphics[width=0.85\textwidth]{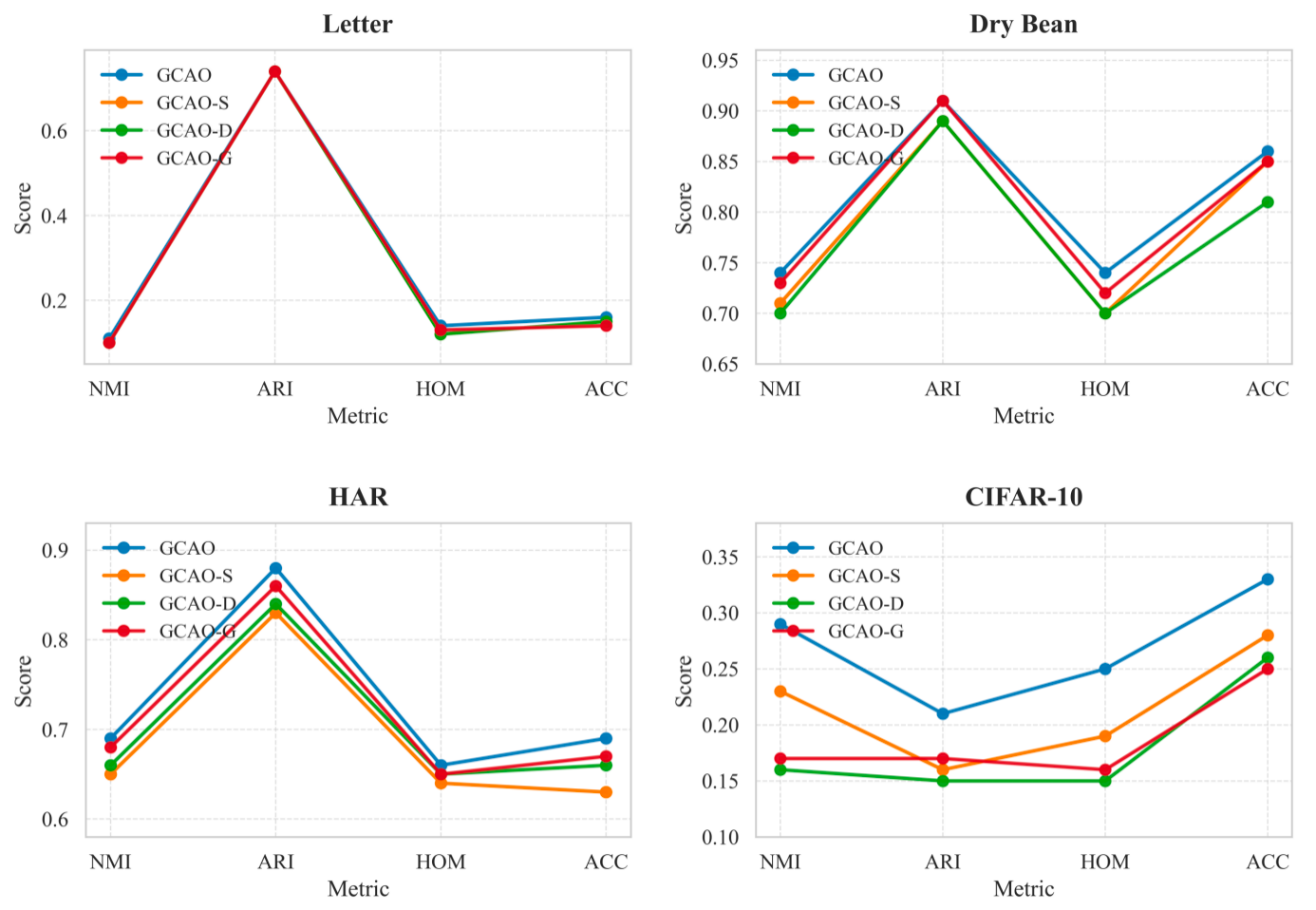}
    \caption{Ablation study visualization: clustering performance variation under different module configurations.}
    \label{fig:ablation_metric}
\end{figure*}

\begin{table*}[htbp]
\centering
\caption{Ablation Experiment Results: Impact of Different Modules on Clustering Performance}
\label{tab:ablation}
\renewcommand{\arraystretch}{1.2}
\setlength{\tabcolsep}{3.5pt}
\begin{tabular}{l|cccc|cccc|cccc|cccc}
\toprule
\multirow{2}{*}{Algorithm} & \multicolumn{4}{c|}{HAR} & \multicolumn{4}{c|}{Dry Bean} & \multicolumn{4}{c|}{Letter} & \multicolumn{4}{c}{CIFAR-10} \\
\cmidrule(lr){2-5} \cmidrule(lr){6-9} \cmidrule(lr){10-13} \cmidrule(lr){14-17}
& NMI & ARI & HOM & ACC & NMI & ARI & HOM & ACC & NMI & ARI & HOM & ACC & NMI & ARI & HOM & ACC \\
\midrule
GCAO & \textbf{0.69} & \textbf{0.88} & \textbf{0.66} & \textbf{0.69} & \textbf{0.74} & \textbf{0.91} & \textbf{0.74} & \textbf{0.86} & \textbf{0.11} & \textbf{0.74} & \textbf{0.14} & \textbf{0.16} & \textbf{0.29} & \textbf{0.21} & \textbf{0.25} & \textbf{0.33} \\
GCAO-S & 0.65 & 0.83 & 0.64 & 0.63 & 0.71 & 0.89 & 0.70 & 0.85 & 0.10 & 0.74 & 0.12 & 0.15 & 0.23 & 0.16 & 0.19 & 0.28 \\
GCAO-D & 0.66 & 0.84 & 0.65 & 0.66 & 0.70 & 0.89 & 0.70 & 0.81 & 0.10 & 0.74 & 0.12 & 0.15 & 0.16 & 0.15 & 0.15 & 0.26 \\
GCAO-G & 0.68 & 0.86 & 0.65 & 0.67 & 0.73 & 0.91 & 0.72 & 0.85 & 0.10 & 0.74 & 0.13 & 0.14 & 0.17 & 0.17 & 0.16 & 0.25 \\
\bottomrule
\end{tabular}
\end{table*}

\subsection{Ablation Experiments}

To verify the effectiveness and necessity of each key module proposed in this paper, we conducted systematic ablation experiments on the three core components: \emph{cooperative moving groups}, \emph{local density estimation}, and \emph{group force response}. Specifically, three variant models were constructed: A) removing the cooperative moving group mechanism, relying solely on single-point forces for movement; B) disabling low-density point selection so that all data points participate in contraction; C) removing the group force response mechanism and using uniform weights for force calculation. The clustering performance of these variants was evaluated on multiple real-world datasets, using comprehensive metrics including NMI, ARI, HOM, and ACC.

\textbf{Experimental Design.}  
The cooperative moving group mechanism aims to suppress random displacements of individual points in low-density regions through group-level coordination, enabling a more stable contraction process at the group level. Local density estimation updates only low-density points, fixing the high-density core regions and maintaining the overall structural steady state. The group force response mechanism calculates forces between points using distance-based weighting, improving the rationality of force distribution and local balance in neighborhoods. Removing any of these modules may lead to performance degradation in clustering.

\textbf{Results and Analysis.}  
Table \ref{tab:ablation} and Fig. \ref{fig:ablation_metric} show the performance of the four algorithm variants on different datasets. The variants are: GCAO, GCAO-S (removing the cooperative moving group mechanism, relying solely on single-point forces), GCAO-D (disabling low-density point selection so that all points participate in contraction), and GCAO-G (removing the group force response mechanism, using uniform weights for force calculation). It can be observed that the complete GCAO model outperforms all its ablated versions across all metrics. Specifically, removing the cooperative moving group mechanism causes an average decrease of about 5\% in NMI and ARI, indicating the key role of this mechanism in maintaining intra-cluster consistency and boundary stability. Disabling low-density point selection leads to excessive movement of low-density points, resulting in blurred boundaries, especially noticeable on the CIFAR-10 dataset. Removing the distance-based weighting in the group force response has a smaller impact but still performs slightly lower than the complete model, demonstrating the positive effect of the weighting strategy in maintaining local force balance. Overall, the collaborative design of these three mechanisms ensures GCAO's stability and robustness under complex data distributions.

\textbf{Summary.}  
From the comprehensive experimental results, it is evident that the cooperative moving group mechanism contributes the most to preventing boundary drift and enhancing structural consistency; density-based selection effectively distinguishes core and boundary regions, avoiding structural distortion caused by excessive contraction of low-density points; the weighting strategy further balances local force relations, ensuring smoothness and stability during the optimization process. The synergy of these three mechanisms enables GCAO to demonstrate superior stability and generalization across a variety of complex datasets.

\subsection{Hyperparameter Discussion}

In the GCAO algorithm, the core hyperparameters include \textbf{neighborhood size $k$}, \textbf{step size coefficient $\lambda$}, and \textbf{maximum number of iterations $T$}. Specifically, $k$ controls the local density perception range for group construction, $\lambda$ determines the magnitude of each group contraction, and $T$ governs the overall progression of the force optimization. These three parameters jointly determine the clustering performance and convergence characteristics of the algorithm. Detailed sensitivity analysis of these parameters was conducted in the experiments to validate their settings.

Table~\ref{tab:hyperparameters} shows the optimal hyperparameter combinations of GCAO on each dataset.

\begin{table}[htbp]
    \centering
    \caption{Optimal GCAO hyperparameters for each dataset ($k$, $\lambda$, $T$)}
    \label{tab:hyperparameters}
    \small % 缩小字体，避免超出页面
    \begin{tabular}{lccc}
        \toprule
        Dataset & $k$ & $\lambda$ & $T$ \\
        \midrule
        HAR       & 5  & 0.3 & 7 \\
        Dry Bean  & 4  & 0.7 & 9 \\
        Letter    & 18 & 0.1 & 8 \\
        CIFAR-10  & 10 & 0.1 & 3 \\
        \bottomrule
    \end{tabular}
\end{table}

\textbf{Neighborhood Size $k$.} This parameter defines the neighborhood size used in the GFP to estimate local density. To explore its impact, we tested $k$ in the range $[3, 20]$. When $k$ is too small, group formation is easily affected by local noise, resulting in unstable performance. As $k$ increases, clustering performance (e.g., NMI and ARI) stabilizes and improves. However, as shown in Table~\ref{tab:hyperparameters}, the optimal $k$ is closely related to data characteristics. For datasets with relatively clear structures such as HAR ($k=5$) and Dry Bean ($k=4$), a smaller $k$ suffices. For datasets with many categories and complex structures like Letter ($k=18$) or high-dimensional sparse data like CIFAR-10 ($k=10$), a larger $k$ is needed to capture sufficiently robust neighborhood information. When $k$ exceeds a certain threshold (e.g., $k>20$), performance tends to plateau or slightly decrease, while computational cost increases.

\textbf{Step Size Coefficient $\lambda$.} This parameter determines the magnitude of group contraction in the GOP. We tested multiple values of $\lambda$ within the range [0.1, 2.0]. The choice of $\lambda$ reflects the trade-off between ``convergence speed'' and ``clustering accuracy.'' A smaller $\lambda$ makes the contraction process smoother, facilitating fine adjustments along complex boundaries and preventing premature merging of clusters, thus achieving optimal results for complex datasets such as Letter and CIFAR-10 ($\lambda=0.1$). A larger $\lambda$ accelerates convergence and performs well on datasets with relatively clear cluster structures, such as Dry Bean ($\lambda=0.7$). However, if $\lambda$ is too large (e.g., $\lambda>1.5$), the group movement may overshoot the optimal positions, reducing clustering accuracy.

\textbf{Maximum Iterations $T$.} This parameter controls the total number of iterations for the force optimization process. We tested $T$ in the range $[1, 10]$. The algorithm's performance improves rapidly with increasing $T$ and gradually converges after about 5 iterations. The optimal $T$ values in Table~\ref{tab:hyperparameters} all fall within this range, indicating that GCAO converges quickly and does not require a large number of iterations.

\section{CONCLUSION AND FUTURE WORK}

In this paper, we proposed a Group-driven Clustering via Gravitational Attraction and Optimization algorithm. By treating groups as the fundamental moving units and integrating a neighborhood construction mechanism guided by local density, GCAO enables the coordinated contraction of group members. This approach effectively addresses the issues in traditional point-to-point contraction methods, such as cluster splitting, blurred boundaries, and interference from outliers in high-dimensional, complex-boundary, and multi-density data environments. During the movement process, GCAO introduces a dual-force model that considers both local attraction from outside the group and cohesive forces within the group to enhance stability, thereby significantly improving cluster structure preservation while ensuring convergence.

Experimental results demonstrate that on multiple real-world datasets, GCAO consistently outperforms traditional density-peak clustering, gravitational contraction clustering, and related swarm-intelligence-based methods in terms of clustering metrics such as NMI and ARI. It exhibits particularly high robustness and accuracy in scenarios with complex cluster boundaries and high noise interference. Although GCAO introduces several hyperparameters, including group partitioning, force weighting, and step size control, these parameters show strong adaptability in experiments, and their sensitivity can be further reduced through density-adaptive mechanisms. 

In future work, we plan to explore high-dimensional acceleration strategies, adaptive optimization of group force models, and integration with deep feature learning, aiming to enhance the scalability and applicability of GCAO on large-scale heterogeneous data. Overall, GCAO not only provides theoretical support and robustness analysis for group-based contraction clustering but also demonstrates excellent practical clustering performance in complex data environments, offering an efficient, robust, and interpretable approach for unsupervised analysis of high-dimensional data.

\end{sloppypar}
\end{document}